\documentclass[10pt]{proc}%
\usepackage{fullpage}
\usepackage[textwidth=7in,textheight=10in]{geometry}

\usepackage{cite}

\usepackage{pgfplots}

\usepackage{amsthm,amsmath,amssymb,amsfonts}
\usepackage{graphicx}
\usepackage{textcomp}
\usepackage{url}

\usepackage{flushend}
\usepackage[linesnumbered,algoruled,boxed,lined]{algorithm2e}

\newcommand{\ie}{\emph{i.e.}, }%
\newcommand{\eg}{\emph{e.g.}, }%

\usepackage[makeroom]{cancel}

\newcommand{\erwan}[1]{\textcolor{black}{#1}}
\newcommand{\gilles}[1]{\textcolor{black}{#1}}

\usepackage{subfigure}

\newtheorem{definition}{Definition}%
\theoremstyle{plain}

\newtheorem{lem}{Proposition}

\newtheorem{obs}{Observation}

\usepackage{ color, colortbl}
\definecolor{Gray}{gray}{0.9}

\usepackage{tikz}
\usetikzlibrary{arrows, arrows.meta, positioning,shapes}

\title{
The Bouncer Problem: Challenges to Remote Explainability
}

\author{
  Erwan {Le Merrer}\\
     Univ Rennes, Inria, CNRS, Irisa\\
       erwan.le-merrer@inria.fr 
  \and Gilles Tr\'edan\\
       LAAS/CNRS\\
       gtredan@laas.fr
}

\begin{document}

\date{}

\maketitle

 \begin{abstract}
   The concept of explainability is envisioned to satisfy society's
   demands for transparency on machine learning decisions. The concept
   is simple: like humans, algorithms should explain the rationale
   behind their decisions so that their fairness can be
   assessed.

   While this approach is promising in \erwan{a local context (\eg the model creator
   explains it during debugging at training time)}, we argue that
   this reasoning cannot simply be transposed in a remote context,
   where a trained model by a service provider is only accessible
   \erwan{to a user through a network and its API}. This is problematic as it constitutes precisely
   the target use-case requiring transparency from a societal
   perspective.

   Through an analogy with a \textit{club bouncer} (which may provide
   untruthful explanations upon customer reject), we show that
   providing explanations cannot prevent a remote service from lying
   about the true reasons leading to its decisions.
   More precisely, we prove the impossibility of remote explainability
   for single explanations, by constructing an attack on
   explanations that hides discriminatory features to the querying user.
   
   We provide an example implementation of this attack. We then show
   that the probability that an observer spots the attack, using
   several explanations for attempting to find incoherences, is low in practical settings. This undermines
   the very concept of remote explainability in general.

\end{abstract}

\section{Introduction}

Modern decision-making driven by black-box systems now impacts a
significant share of our lives \cite{fatml,deLaat2018}.  These systems
build on user data, and range from recommenders \cite{fb-rec} (\eg for
personalized ranking of information on websites) to predictive
algorithms (\eg credit default) \cite{fatml}. This widespread deployment, along
with the opaque decision process provided by these systems
raises concerns about transparency for the general public or for policy
makers \cite{Goodman_Flaxman_2017}.
This translated in some jurisdictions (\eg United States of America and Europe) into a so called \textit{right to explanation} \cite{Goodman_Flaxman_2017, 10.1093/idpl/ipx022}, that states that the output decisions of an algorithm must be motivated.

\paragraph{Explainability of in-house models}
An already large body of work is interested in the
\textit{explainability} of implicit machine learning models (such as
neural network models)
\cite{adadi2018peeking,guidotti2018survey,molnar2019}.  Indeed,
these models show state-of-art performance when it comes to a task
accuracy, but they are not designed to provide explanations --or at
least intelligible decision processes-- when one wants to obtain more
than the output decision of the model.  In the context of
\textit{recommendation}, the expression ``post hoc explanation'' has
been coined \cite{rec-exp}.  In general, current techniques for
explainability of implicit models take trained in-house models and aim
at shedding light on some input features causing salient decisions in
their output space.  LIME \cite{lime} for instance builds a surrogate
model of a given black-box system that approximates its predictions around a region of interest.
\erwan{The surrogate is created from a new crafted dataset,
obtained from the permutation of the original dataset values around
the interesting zone (and the observation of the decisions made
for this dataset).}
This surrogate is an explainable model by construction (such as a
decision tree), so that it can explain some decision facing some input
data. The amount of queries to the black-box model is assumed to be
unbounded by LIME and others \cite{Galhotra:2017:FTT:3106237.3106277,shap},
permitting virtually exhaustive queries to it. This reduces
  their applicability to the inspection of in-house models by their
  designers.
  
\paragraph{The temptation to explain decisions to users.}
\erwan{As suggested by Andreou et
al. \cite{andreou2018ndss}, some institutions may apply the same reasoning in order
to explain some decisions to their users}. Indeed, this would
support the will for a more transparent and trusted web by the
public. Facebook for instance attempted to offer a form of
transparency for the ad mechanism targeting its users, by introducing
a ``Why I am seeing this'' button on received ads.  For a user, the
decision-making system (here, responsible of selecting relevant ads) is then \textit{remote}, and can be queried only using
inputs (its profile data) and the observation of system decisions.
Yet, from a security standpoint, \erwan{we consider a security model where the remote server (executing the service) is untrusted to the users, in the classic remote execution setup \cite{pdp}.}  Andreou et
al. \cite{andreou2018ndss} recently empirically observed in the case
of Facebook that these explanations are ``incomplete and can be
misleading'', conjecturing that malicious service providers can use
this incompleteness to hide the true reasons behind their decisions.

In this paper, we question the possibility of such an explanation
setup, from a corporate and private model in destination to users: we
go one step further%
by demonstrating that remote
explainability simply cannot be a \erwan{reliable guarantee of the lack of
 discrimination in the decision-making process}. In a remote black-box setup
such as the one of Facebook, we show that a simple attack, we coin the
Public Relations (PR) attack, undermines remote explainability.

\paragraph{The bouncer problem as a parallel for hardness}
For the sake of the demonstration, we introduce the \textit{bouncer
  problem} as an illustration of the difficulty for users to spot
malicious explanations. The analogy works as follows: let us picture a
bouncer at the door of a club, deciding whoever might enter the club.
When he issues a negative decision --refusing the entrance to a given
person--, he also provides an explanation for this rejection.
However, his explanation might be malicious, in the sense that his
explanation does not present the true reasons of this person's
rejection. Consider for instance a bouncer discriminating people based
on the color of their skin. Of course he will not tell people he
refuses the entrance based on that characteristic, since this is a
legal offence. He will instead invent a biased explanation
that the rejected person is likely to accept.

The classic way to
assess a  discrimination by the bouncer is for associations to run
tests (following the principle of statistical causality
\cite{CIS-247618} for instance): several persons attempt to enter,
while they only vary in their attitude or appearance on the possibly
discriminating feature (\eg the color of their skin).  Conflicting
decisions by the bouncer is then the indication of a possible discrimination and
is amenable to the building of a case for prosecution.

We make the parallel with bouncer decisions in this paper by
demonstrating that a user cannot trust a single (one-shot) explanation provided
by a remote model.\gilles{Moreover, we show that creating such malicious
explanations necessarily creates inconsistent answers for some inputs, and that} the only solution to spot those inconsistencies
is to issue multiple requests to the service. Unfortunately, we also
demonstrate the problem to be hard, in the sense that spotting an
inconsistency in such a way is intrinsically not more efficient than
for \erwan{a model creator to exhaustively search on her local model to identify a
problem}, which is often considered as an intractable process.

\paragraph{Rationale and organization of the paper}
We build a general setup for remote explainability in the next section,
that has the purpose of representing actions by a service provider and
by users, facing models decisions and explanations. The fundamental
 blocks for the impossibility proof of a reliable remote
explainability, or its hardness for multiple queries are presented in
Section \ref{s:model}. We present the \textit{bouncer problem} in
Section \ref{s:bouncer}, that users have to solve in order to detect
malicious explanations by the remote service provider. We then
illustrate the PR attack, that the malicious provider may execute to
remove discriminative explanations to users, on decision trees
(Section \ref{s:dt}). We then practically address the bouncer problem
by modeling a user trying to find inconsistencies from a provider
decisions based on the German Credit Dataset and a neural network
classifier, in Section \ref{s:nn}. We discuss open problems in Section
\ref{s:discussion}, %
  before reviewing related works in Section \ref{s:related} and
concluding in Section \ref{s:conclusion}.  Since we show that remote
explainability in its current form is undermined, this work thus aims
to be a motivation for researchers to explore the direction of
\textit{provable} explainability, by designing new protocols such as
for instance one implying cryptographic means (\eg such as in
\textit{proof of ownership} for remote storage), or to build
collaborative observation systems to spot inconsistencies and malicious
explanation systems.

\section{Explainability of remote decisions}
\label{s:model}

In this work, we study \textit{classifier models}, that will issue
decisions given user data. We first introduce the setup we operate
in: it is intended to be as general as possible, so that the results
drawn from it can apply widely.

\subsection{General Setup}
 We consider a classifier $C: \mathcal{X}\mapsto \mathcal{Y}$ that
 assigns inputs $x$ of the feature space $\mathcal{X}$ to a class
 $C(x)=y\in \mathcal{Y}$. Without loss of generality and to simplify
 the presentation, we will assume the case of a binary classifier:
 $\mathcal{Y}=\{0,1\}$; the decision is thus the output label
 returned by the classifier.

\paragraph{Discriminative features and classifiers}
To produce a decision, classifiers rely on features (variables) as an
input. These are for instance the variables associated to a user
profile on a given service platform (\eg basic demographics, political
affiliation, purchase behavior, residential profile
\cite{andreou2018ndss}).  In our model, we consider that the feature
space contains two types of features: \emph{discriminatory} and
\emph{legitimate} features.  The use of discriminatory features allows
for exhibiting the possibility of a malicious service provider issuing
decisions and biased explanations. This problematic is also
  referred to as \textit{rationalization} in a paper by A\"ivodji et al
  \cite{pmlr-v97-aivodji19a}.

Concretely, we consider \emph{discriminatory} features to be an
arbitrary subset of the input features, such that we can define these
as "any feature set the malicious service provider does not want to
explain". Two main reasons come to mind:
\begin{itemize}
\item Legal: the jurisdiction's law forbids decisions based on a list
  of criteria\footnote{For instance in the U.K.:
    \url{https://www.gov.uk/discrimination-your-rights}.} which are
  easily found in classifiers input spaces. A service provider risks
  prosecution upon admitting the use of these.

  For instance, features such as age, sex,
  employment, or the status of foreigner are considered as
  discriminatory in the work by Hajian et al. \cite{10.1007/978-3-642-22589-5_20}, that
  looks into the German Credit Dataset, that links bank customer
  features to the accordance or not of a credit.
\item Strategical: the service provider wants to hide the use of some
  features on which its decisions are based. This could be to hide
  some business secret from competitors (because of the accuracy-fairness trade-off \cite{pmlr-v81-menon18a} for instance), or to avoid "reward
  hacking" from users biasing this feature, or to avoid bad press.
\end{itemize}
Conversely, any feature that is not discriminatory is coined \textit{legit}.

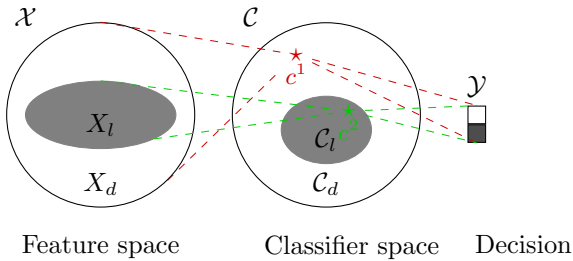
\begin{figure}[t!]
  \begin{tikzpicture}[scale=1]
    \tikzstyle{BC} = [ draw, rectangle, node distance=10pt, minimum
    width=3cm, minimum height=2cm, text width=3cm, align=center, ]
    \tikzstyle{adv} = [ draw, rectangle, node distance=10pt, minimum
    width=3cm, minimum height=0.5cm, text width=3cm, align=center, ]
    \tikzset{edge/.style={ ->,  thick }, }
    \node[ellipse, fill=gray!10, gray, draw, very thin, text height =.4cm, minimum width = 2cm,label={[anchor=south,above=.5mm]270:$X_l$}] at (0,0)  (xl) {};
    \node[ellipse, draw, text height =1.5cm, minimum width = 2.5cm,label={[anchor=south,above=.5mm]270:$X_d$}] at (0,0)  (x) {};
    \node at (-1,1.3) {$\mathcal{X}$};

\node[yshift=-5em]{Feature space};

    \node[ellipse,  fill=gray!10, gray, draw, very thin, text height =.4cm, minimum width = 1.2cm,label={[anchor=south,above=.5mm]270:$\mathcal{C}_l$}] at (3,-.2)  (semi) {};
    \node[ellipse, draw, text height =1.5cm, minimum width = 2.5cm,label={[anchor=south,above=.5mm]270:$\mathcal{C}_d$}] at (3,0)  (semi) {};
    \node at (2,1.3) {$\mathcal{C}$};

\node[xshift=9.5em,yshift=-5em]{Classifier space};
    
    \node[rectangle,draw,fill=white,label=$\mathcal{Y}$] (yes) at (5,0) {};
    \node[rectangle,draw,fill=black!70,anchor=north] (no) at (yes.south) {};

\node[xshift=16em,yshift=-4.9em]{Decision};
    
    \node[red!80!black, label={[red!80!black,below=.2cm] \small $c^1$}]
    (ca) at (2.6,.8)  {$\star$};
        \node[green!80!black, label={[green!80!black,below=.2cm] \small $c^2$}]
    (cb) at (3.3,.05)  {$\star$};

    \draw[dashed,very thin, red!80!black] (x.south east) -- (ca) --
    (no.south);
    \draw[dashed,very thin, red!80!black] (x.north) -- (ca) --
    (yes.north);

    \draw[dashed,very thin, green!80!black] (xl.south east) -- (cb) --
    (no.south);
    \draw[dashed,very thin, green!80!black] (xl.north) -- (cb) --
    (yes.north);
  \end{tikzpicture}
  \caption{Illustration of our model: we consider binary classifiers, that
    map the input domain $\mathcal{X}$ to labels
    $\mathcal{Y}=\{0,1\}$. Some dimensions of the input space are
    discriminative $X_d$, which induces a partition on the classifier
    space. Legitimate classifiers $\mathcal{C}_l$ that do not rely on
    discriminative features to issue a label (in green), while others
    (that is, $\mathcal{C}_d$) can rely on any feature (in red).}
  \label{f:space}
\end{figure}

Formally, we partition the classifier input space $\mathcal{X}$ along
these two types of features: legitimate features $X_l$ that the model
can legitimately exploit to issue a decision, and discriminative
features $X_d$ (please refer to Figure \ref{f:space}).
In other words $\mathcal{X}=(X_l,X_d)$, and any input
$x\in \mathcal{X}$ can be decomposed as a pair of legitimate and
discriminatory features $x=(x_l,x_d)$. \erwan{We stress that the introduction of such a split in the features is required to build our proof and studies, yet it does not constitute a novel proposal in any way.} We assume the input contains at
least one legitimate feature: $X_l\neq \emptyset$.

We also partition the classifier space accordingly: let $\mathcal{C}_l
\subset \mathcal{C}$ the space of legitimate classifiers (among all
classifiers $\mathcal{C}$), which do not rely on any feature of $X_d$
to issue a decision. More precisely, we consider that a classifier is
legitimate if and only if arbitrarily changing any discriminative
input feature does never change its decision: $$C \in \mathcal{C}_l
\Leftrightarrow \forall x_l \in X_l, \forall x_d,x_d' \in
\mathcal{X}_d^2, C((x_l,x_d))=C((x_l,x_d')).$$  Observe that therefore,
any legitimate classifier $C_l$ could simply be defined over input
subspace $X_l\subset \mathcal{X}$. As a slight notation abuse to stress that the
value of discriminative features does not matter in this legitimate
context, we write $C((x_l,\emptyset))$, or $C(x \in X_l)$ as the
decision produced regardless of any discriminative feature.  It
follows that the space of discriminative classifiers complements the
space of legitimate classifiers: $\mathcal{C}_d = \mathcal{C}\setminus
\mathcal{C}_l$.

We can now reframe the main research question we address: \textit{Given a set
of discriminative features $X_d$, and a classifier $C$, can we decide
if $C \in \mathcal{C}_d$, in the remote black-box interaction model ?}

\paragraph{The remote black-box interaction model}
We question the \emph{remote black-box interaction} model (see \eg paper \cite{Tramer:2016:SML:3241094.3241142}),
where the classifier is exposed to users through a remote API. In
other words users can \emph{only} query the classifier model with an input and
obtain a label as an answer (\eg 0 or 1). In this remote setup, users
then cannot collect any specific information about the internals
of the classifier model, such as its architecture, its weights, or its
training data. This corresponds to a security threat model where two
parties are interacting with each other (the user and the remote
service), and where the remote model is implemented on a server,
belonging to the service operator, that is untrusted by the user.

\subsection{Requirements for Remote Explainability}
Explainability is often presented as a solution to increase the
acceptance of AI \cite{adadi2018peeking}, and to potentially prevent discriminative AI
behaviour. Let us expose the logic behind this connection.

\paragraph{Explanations using conditional reasoning}
First, we need to define what is an explanation, to go beyond Miller's definition as
an ``answer to a why-question'' \cite{DBLP:journals/corr/Miller17a}. %
Since the topic of explainability is becoming a hot research field
with (to the best of our knowledge) no consensus on a more technical
definition of an explanation, we will propose for the sake of our
demonstration that an explanation is \erwan{causally coherent, with
  respect to the \emph{modus ponens} %
  rule from deductive reasoning
(it stands for ``if A is implying B, and A being true, B is true as
  well''}) \cite{ponens}.  For instance, if explanation $a$ explains
decision $b$, it means that in context $a$, the decision produced will
necessarily be $b$.  In this light, we first directly observe the
beneficial effect of such explanations on our parallel to club
bouncing: while refusing someone, the bouncer may provide him with the
reasons of that rejection; the person can then correct their behaviour
in order to be accepted on next attempt.

\erwan{Second, this modus ponens explanation form is also sufficient
  to prove non-discrimination. For instance, if $a$ does not involve
  discriminating arguments (which can be checked by the user as $a$ is
  a sentence), and $a\Rightarrow b$, then decision $b$ is not
  discriminative in case $a$.} On the contrary, if $a$ does involve
discriminating arguments, then decision $b$ is taken on a
discriminative basis, and is therefore a discriminative decision. In
other words, this property of an explanation is enough to reveal
discrimination.

To sum up, any explanation framework that behaves ``logically'' (\ie
fits the modus ponens \cite{ponens}) --which is in our view a rather
mild assumption-- is enough to establish the discriminative basis of a
decision. We believe this is the rationale of the statement
"transparency can improve users trust in AI systems".
\gilles{In fact, this logical behaviour is not only sufficient
  to establish discrimination, it is also necessary: assume a
  framework providing explanation $a$ for decision $b$ such that
  we do not have $a\Rightarrow b$. Since $a$ and $b$ are not connected
  anymore, $a$ does not bring any information about $b$.}

\gilles{While this logical behaviour is desirable for users,
  unfortunately in a remote context they cannot check whether
  $a\Rightarrow b$ is in general true because they are only provided
  with a particular explanation $a$ leading to a particular decision
  $b$. They cannot check that $a$ being true leads \emph{in all
    contexts} to $b$ being true. }

\paragraph{Requirements on the user side for checking explanations}
\gilles{In a nutshell, a user in a remote interaction can verify that
  in her context $a$ is true, and $b$ is true, which is compatible
  with the $a\Rightarrow b$ relation of an explanation fitting the
  modus ponens. Let us formalise what can a user check regarding the
  explanation she collects.}  A user that queries a classifier $C$
with an input $x$ gets two elements: the decision (inferred class)
$y=C(x)$ and an explanation $a$ such that $a$ explains $y$. Formally,
upon request $x$, a user collects $y$ and $a = exp_C(y,x)$.

We assume that such a user can check that $a$ is \emph{apropos} \erwan{(\ie appropriate)}: $a$
corresponds to her input $x$. Formally, we write $a\in A(x)$. This
allows us to formally write a non-discriminatory explanation as
$a\in A(x_l)$. This forbids lying by explaining an input that is
different than $x$.

We also assume that the user can check the explanation is
\emph{consequent}: user can check that $a$ is compatible with
$y$. This forbids crafting explanations that are incoherent w.r.t. the
decision (like a bouncer that would explain why you can enter in while
leaving the door locked).

To produce such explanations, we assume the existence of an
explanation framework $exp_C$ producing explanations for classifier
$C$ (this could for instance by the LIME framework \cite{lime}). The
explanation $a$ explaining decision $y$ in context $x$ by classifier
$C$ is written $a=exp_C(y,x)$.
 
Having defined the considered model for exposing our results,
  we stress that this model aims at constraining the provider as much
  as possible (\eg explanations must be as complete as possible, are
  always provided, must always be coherent with decision, etc.). The
  intuition being that if we prove the possibility of malicious
  explanations in this constrained case, then the implementation in
  all less constrained cases, such as for incomplete explanation
  \cite{andreou2018ndss}, \gilles{or example-based explanations} will
  only be easier.

  \gilles{To sum up on the explanation model: explanations fitting the
    modus ponens allow users to detect discrimination. Unfortunately
    in a remote context, users cannot check whether explanations do
    fit the modus ponens. However they can check the veracity of the
    explanation and the decision in their particular experience. This
    is the space we exploit for our attack, by generating malicious
    explanations that appear correct to the user (yet that do not fit
    the modus ponens).}

\subsection{Limits of Remote Explainability: The PR (Public Relations) Attack}

We articulate our demonstration of the limits of explainability in a
remote setup by showing that a malicious service provider can hide the
use of discriminating features for issuing its decisions, while
conforming to the mild explainability framework we described in
the previous subSection.

Such a malicious provider thus wants to \emph{i)} produce decisions
based on discriminative features and to \emph{ii)} produce
non-discriminatory explanations to avoid prosecution.

\gilles{A first approach could be to manipulate the explanation
  directly. It might however be difficult to do so while keeping the
  explanation convincing and true in an automated way. In this paper,
  we follow another strategy that instead consists in inventing a
  legitimate classifier that will then be explained.}

\paragraph{A Generic Attack Against Remote Explainability}
We coin this attack the \emph{Public Relations attack} (noted PR). The idea is
rather simple: upon reception of an input $x$, first compute
discriminative decision $C(x)$. Then train a surrogate model
$C'$ that is non-discriminative, and such that $C'(x)=y$.  Explain
$C'(x)$, and return this explanation along with $C(x)$.

\begin{figure*}[h]
  \begin{tikzpicture}[scale=1]
    \tikzstyle{BC} = [ draw, rectangle, node distance=10pt, minimum
    width=3cm, minimum height=2cm, text width=3cm, align=center, ]
    \tikzstyle{adv} = [ draw, rectangle, node distance=10pt, minimum
    width=3cm, minimum height=0.5cm, text width=3cm, align=center, ]
    \tikzset{edge/.style={ ->,  thick }, }
  
    \node[BC] (0,0)
    {{$x=(x_l,\emptyset)\:\stackrel{C}{\rightarrow}\:y$}\\
      [1em]\Large{Remote}};

    \draw[edge] (-0.5,-3) --  node[left] {\parbox{1cm}{\hfill$x$}} ++(0,+2);
    \draw[edge] (+0.5,-1) -- node[right] {\parbox{1cm}{$y,exp_C(y,(x_l,\emptyset))$}}
    ++(0,-2);
    \node[adv] at (0,-3.3) {User};
    \node at (-2,1.5) {\textbf{A.}};

    \begin{scope}[xshift=160]
    \node[BC] (0,0)
    {{$x=(x_l,x_d)\:\stackrel{C}{\rightarrow}\:y$}\\
      [1em]\Large{Remote}};

    \draw[edge] (-0.5,-3) --  node[left] {\parbox{1cm}{\hfill$x$}} ++(0,+2);
    \draw[edge] (+0.5,-1) -- node[right] {\parbox{1cm}{$y,exp_C(y,(x_l,\textcolor{red}{x_d}))$}}
    ++(0,-2);
    \node[adv] at (0,-3.3) {Discriminated User};
    \node at (-2,1.5) {\textbf{B.}}; 
  \end{scope}

    \begin{scope}[xshift=320]
      \node[BC] (0,0)
      {{$(x_l,x_d)\:\stackrel{C}{\rightarrow}\:y$}\\
          {{\scriptsize $\text{PR}(C,(x_l,x_d),y) \rightarrow \textcolor{green}{C'}$}\\
        s.t. $C'(x_l)=y$}};  
      
      \draw[edge] (-0.5,-3) --  node[left] {\parbox{1cm}{\hfill$x$}} ++(0,+2);
      \draw[edge] (+0.5,-1) -- node[right] {\parbox{1cm}{$y,exp_{\textcolor{green}{C'}}(y,(x_l,\emptyset))$}}
      ++(0,-2);	
      \node[adv] at (0,-3.39) {Discriminated \& Fooled User};       
    \node at (-2,1.5) {\textbf{C.}}; 
    \end{scope}

  \end{tikzpicture}
  \caption{(\textbf{A.}) %
    A provider using a model that does not leverage discriminatory features.
    (\textbf{B.}) A discriminative model divulges its use of a discriminating feature.
    (\textbf{C.}) The PR attack
    principle, undermining remote explainability: the black-box \erwan{builds a surrogate model} $C'$ for each new request $x$,
    that decides $y$ based on $x_l$ features only. It explains $y$ using $C'$.}
  \label{fig:attack}
\end{figure*}
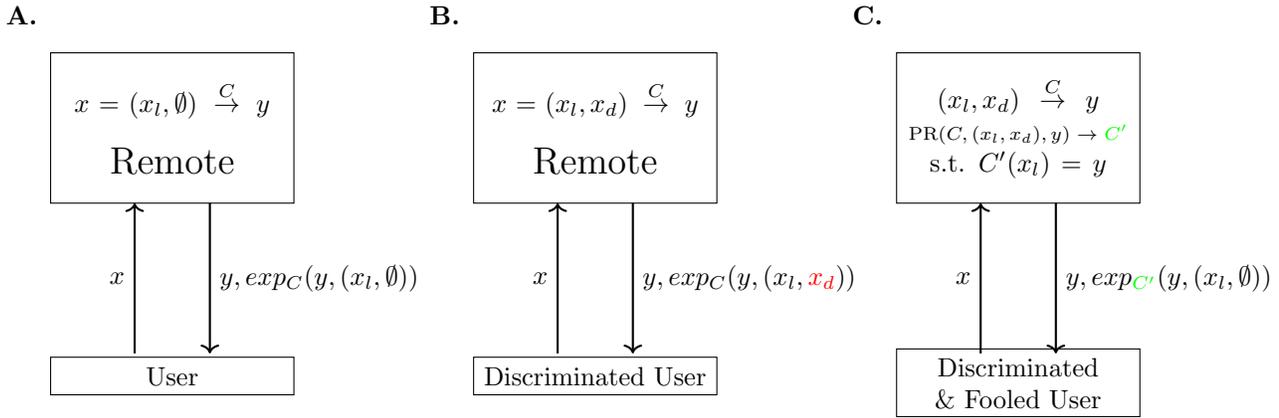

Figure~\ref{fig:attack} illustrates a decision based solely on legitimate features %
(\textbf{A.}), a provider giving an explanation that includes discriminatory features (\textbf{B.}), and the attack by a
malicious provider (\textbf{C.}).
In all three scenarios, a user is
querying a remote service with inputs $x$, and obtaining decisions
$y$ each along with an explanation.
In case \textbf{B.}, the explanation $exp_C$ reveals the use of
discriminative features $X_d$; this provider is prone to
complaints. To avoid these, the malicious provider (\textbf{C.}) leverages the PR attack, by first
computing $C(x)$ using its discriminative classifier $C$. Then, based on the
legitimate features $x_l$ of the input, and its final (discriminative) decision $y$,
it derives a classifier $C'$ for the explanation. %

Core to the attack is the ability to derive such classifier $C'$:
\begin{definition}[PR attack]
  Given an arbitrary classifier $C\in \mathcal{C}_d$, a PR attack is a
  function that finds for an arbitrary input $x$ a classifier $C'$:
  \begin{equation}
    \text{PR}(C,x,C(x)) \rightarrow C',
    \label{PR}
  \end{equation}
such that $C'$ satisfies two properties:
  \begin{itemize}
  \item \textbf{coherence:} $C'(x_l)=y$.
  \item \textbf{legitimacy:} $C'\in \mathcal{C}_l$.
  \end{itemize}
\end{definition}
\gilles{Informally, coherence ensures that
  the explanation (derived from $C'$) appears consequent to the user
  observing decision $y$, while legitimacy ensures the explanation
  will appear to the user as originating from the "modus ponens" explanation of a
  non-discriminating classifier. }

\textbf{Effectiveness of the attack:} \gilles{Let us consider the
  perspective of a user whom upon request $x$ collects a $y$ answer along
  with an explanation $a$}. Observe that  $a=exp_{C'}(y,x)$
is apropos since it directly involves $x: a \in A(x)$. Since we have
$C'(x)=y$ it is also consequent. Finally, observe that since $C'\in \mathcal{C}_l$,
then $a\in A(x_l)$: $a$ is non-discriminatory. \gilles{So from the user
  perspective, she collects an apropos and consequent explanation that
could originate from the logical explanation of a legitimate classifier.} %

\textbf{Existence of the attack:}
We note that crafting a classifier $C'$ satisfying the first property
is trivial since it only involves a single data point $x$.
An example solution is the Dirac delta function of the form: 
\[   
C'(x') =C'((x'_l,x'_d))= 
     \begin{cases}
       \delta_{x'_l,x_l} &\quad\text{if } y=1\\
       1- \delta_{x'_l,x_l} &\quad \text{if } y=0 \\
     \end{cases},
   \]
   \noindent where $\delta$ is the Dirac delta function.  \gilles{   Informally,
     this solution corresponds to defining the classifier that would
     only answer "bounce" to this specific input $x$, and
     answer "enter" to any other input. 
   }
   
   \gilles{While a corresponding intuitive explanation could be
     "because it is specifically you", explaining this very
     specific function might not fit any explainability
     framework. To alleviate this concern, we provide an
     example implementation of a PR attack that produces legitimate decision
     trees from discriminating ones in section \ref{s:dt}.}

   \gilles{Dirac here only constitutes an example
     proving the existence  of PR attack functions. It is important 
   to realise that many such $C'$ PR attack functions exists: 
   any function $X_l\mapsto \mathcal{Y}$ that satisfies one easy
   condition: $C'(x)=y$. }

 \begin{lem}
   Let $\mathcal{C}_l:X_l\mapsto \mathcal{Y}=\{0,1\}$ the set of all possible
   legit classifiers. Let $\mathcal{PR} \subset \mathcal{C}_l$ be the set of
   possible PR attack functions. We have $|\mathcal{PR}|=|\mathcal{C}_l|/2$:
   half of all possible legit classifiers are PR attack functions.
 \end{lem}
 \begin{proof}
   Pick $x_l\in X_l$ and $y=C(x)$ the decision with which our PR
   attack function must be coherent. Since $\mathcal{C}_l$ is a set
   of functions defined over $X_l$, any particular function $C$ in
   $\mathcal{C}_l$ is defined at $x_l$. Let us partition the function
   space $\mathcal{C}_l$ according to the value these functions take
   at $x_l$: let $\mathcal{A}:\{C \in \mathcal{C}_l $ s.t. $C(x_l) = y \}$ and
   $\mathcal{B}:\{C \in \mathcal{C}_l $ s.t. $C(x_l) = \bar{y} \}$. We have
   $\mathcal{C}_l=\mathcal{A}\cup \mathcal{B}$.

   Let $not:\mathcal{A} \mapsto \mathcal{B}$ be a "negation function" that associate for each
   function $C\in \mathcal{A}$ its negation $not(C) \in \mathcal{B}$
   s.t. $not(C)(x)=1-C(x)$.  Observe that $not \circ not=Id$: $not$ defines a
   bijection between $\mathcal{A}$ and $\mathcal{B}$ (any function in $A$ has exactly one
   unique corresponding function in $\mathcal{B}$ and vice versa). Since $not$ is a bijection,
   we deduce $|\mathcal{A}|=|\mathcal{B}|=|\mathcal{C}_l|/2$.

   Since $\mathcal{A}$ contains all possible legitimate functions
   ($\mathcal{A} \subset
   \mathcal{C}_l$) that are coherent with $C(x_l)=y$, $\mathcal{A}=\mathcal{PR}$. Thus $|\mathcal{PR}|=|\mathcal{C}_l|/2$
 \end{proof}
 
 \gilles{In other words, PR attack functions are easy to find: if one
   could sample $\mathcal{C}_l$ uniformly at random, since $C'(x)=y$
   is equally likely than $C'(x)=\bar{y}$, each sample would yield a
   PR attack function with probability $1/2$. }

We have presented the framework and an attack necessary to question
the possibility of remote explainability. We next discuss the
possibility for a user to spot that an explanation is malicious and obtained by a
PR attack. We stress that if a user cannot, then the very concept of
remote explainability is at stake.

\section{The bouncer problem: spotting PR attacks}
\label{s:bouncer}

We presented in the previous section a general setup for remote
explainability.  We now formalise our research question regarding the
possibility of a user to spot an attack in that setup.

\begin{definition}[The bouncer problem ($BP$)]
  Using $\epsilon$ requests that each returns a decision $y_i=C(x_i)$ and an explanation $exp_C(y_i,x)$, we denote by $BP(\epsilon)$, decide if $C \in \mathcal{C}_d$. %
\end{definition}

\subsection{An Impossibility Result for One-Shot Explanations}

We already know that using a single input point is insufficient:
\begin{obs}
  $BP(1)$ has no solution.
\end{obs}
  \begin{proof}
    The Dirac construction above always exists.
  \end{proof}

Indeed, constructions like the introduced Dirac function, or the tree
pruning construct a PR attack that produces explainable decisions. Given a
single explanation on model $C'$ (\ie $\epsilon=1$) the user cannot distinguish between
the use of a model ($C$ in case \textbf{A.}), or the one of a
crafted model by a PR attack ($C'$ in case \textbf{C.}), since it is consequent.  This
means that such a user cannot spot the use of hidden discriminatory
features due to the PR attack by the malicious provider.

We observed that a user cannot spot a PR attack, with $BP(1)$.  This is
already problematic, as it gives a formal proof on why Facebook
ad explanation system cannot be trusted \cite{andreou2018ndss}.

\subsection{The Hardness of Multiple Queries for Explanation}

To address the case $BP(\epsilon>1)$, we observe that a PR attack
generates a new model $C'$ for each request; in consequence, an
approach to detect that attack is to detect the impossibility (using
multiples queries) of a \emph{single} model $C'$ to produce coherent
explanations for a set of observed decisions. We here study this approach.

Interestingly, classifiers and bouncers share this property that their
outputs are all mutually exclusive (each input is mapped to exactly one
class). Thus we have $Enter \Rightarrow \overline{Bounce}$ (with $Enter$ and $Bounce$ the positive or negative decision to for instance enter a place). In which case it is
impossible to have $a\Rightarrow Enter $ \emph{and} $a\Rightarrow
Bounce $. \gilles{Note that this relation
  assumes a "logical" explainer. On a non logical explainer, since we
  cannot say $a\Rightarrow Enter$ given $a$ and $Enter$, we cannot
  detect such attack.}
Note \gilles{also} that
non-mutually exclusive outputs \gilles{(e.g. in the case of
  recommenders where recommending item $a$ does not imply not recommending item $b$)} are not bound by this rule.

A potential problem for the PR attack is a decision conflict, in which
$a$ could explain both $b$ and $\bar{b}$ its opposite.  For instance,
imagine a bouncer refusing you the entrance of a club because, say,
you have white shoes. Then, if the bouncer is coherent, he should
refuse the entrance to anyone wearing white shoes, and if you witness
someone entering with white shoes, you could argue against the lack of
coherence of the bouncer decisions. We build on those incoherences to
spot PR attacks.

In order to examine the case $BP(\epsilon)$, where $\epsilon>1$, we
first define the notion of an \textit{incoherent pair}:

  \begin{definition}[Incoherent Pair -- IP]
  Let
  $x^1=(x^1_l,x^1_d),x^2=(x^2_l,x^2_d) \in
  \mathcal{X}=X_l\times X_d$ be a two input points in
  the feature space. $x^1$ and $x^2$ form an incoherent pair for
  classifier $C$ iff they both have the same legit feature values in $X_l$ and yet
  end up being classified differently:

$x^1_l=x^2_l \wedge C(x_1)\neq C(x_2)$. For convenience we write
$(x^1,x^2)\in IP_C$.
\end{definition}

Finding such an IP is a powerful proof of PR attack on the model by
the provider. Intuitively, this is a formalization of an
  intuitive reasoning: "if you let others enter with
white shoes then this was not the true reason for my rejection":
\begin{lem}
  Only decisions resulting from a model crafted by a PR attack (\ref{PR}) can exhibit incoherent pairs:
  $ IP_C \neq \emptyset \Rightarrow C \in \mathcal{C}_d$.
\end{lem}
\begin{proof}
  We prove the contra-positive form
  $ C \not\in \mathcal{C}_d \Rightarrow IP_C = \emptyset $. Let
  $C\not\in \mathcal{C}_d$. Therefore $C \in \mathcal{C}_l$, and by
  definition:
  $ \forall x_l \in X_l, \forall x_d,x_d' \in \mathcal{X}_d^2,
  C((x_l,x_d))=C((x_l,x_d'))$. By contradiction assume $IP_C \neq
  \emptyset$. Let $(x^1,x^2) \in  IP_C: x^1_l=x^2_l \wedge C(x_1)\neq
  C(x_2)$. This directly contradicts $C \in \mathcal{C}_l$. Thus
  $IP_C= \emptyset$.
\end{proof}

We can show that there is always a pair of inputs
allowing to detect a discriminative classifier $C \in \mathcal{C}_d$.

\begin{lem}
  A classifier $C'$, resulting from a PR attack, always has at least one incoherent pair:
  $C' \in \mathcal{C}_d \Rightarrow IP_C \neq \emptyset$.
\end{lem}
\begin{proof}
  We prove the contrapositive form $IP_C = \emptyset \Rightarrow C
  \not\in \mathcal{C}_d$. Informally, the strategy here is to prove
  that if no such pair exists, this means that decisions are not based
  on discriminative features in $X_d$, and thus the provider had no
  interest in conducting a PR attack on the model; the considered
  classifier is not discriminating.

  Assume that $IP_C =  \emptyset$. 
  Let $x_\emptyset\in X_d$, and let $C^l:X_l\mapsto
  \mathcal{Y}$ be a legitimate classifier such that
  $C^l(x_l)=C((x_l,x_\emptyset))$.

  Since $IP_C=\emptyset$, this means that
  $\forall x^1,x^2 \in \mathcal{X}, x^1_l=x^2_l \Rightarrow
  C(x_1)=C(x_2)$. In particular 
  $\forall x\in \mathcal{X},
  C(x=(x_l,x_d))=C^l(x_l,x_\emptyset)$. Thus $C=C^l$; by the
  definition of a PR attack being only applied to a model that uses
  discriminatory features, this leads to
  $C\in \mathcal{C} \setminus \mathcal{C}_d$, \ie
  $C\not\in \mathcal{C}_d$.
\end{proof}

Which directly applies to our problem:
\begin{lem}[Detectability lower bound]
  $BP(\vert  \mathcal{X}\vert)$ is solvable.
\label{lower}
\end{lem}

\begin{proof}
Straightforward:   $C' \in \mathcal{C}_d \Rightarrow IP_C \neq
\emptyset$, and since $IP \subseteq \mathcal{X} \times  \mathcal{X}$ testing the whole input space
will necessarily exhibit such an incoherent pair.
\end{proof}

This last result is rather weakly positive: even though any PR attack
is eventually detectable, in practice it is impossible to exhaustively
explore the input space of modern classifiers due to their dimension.
This remark also further questions the opportunity of remote
explainability.

\gilles{Moreover, it is important to observe that while finding an IP
  proves the presence of a PR attack, it is not an efficient technique
  to prove the absence of a PR attack, which is probably the use case
  interesting users the most. Section~\ref{s:nn} details this approach on
  a concrete dataset.}

\gilles{This concludes the theoretical perspective of this paper. To
  sum up, an explainer that could allow to spot classifier
  discrimination should behave logically, this is what is expected by
  the users. However, they can only check the properties of the
  provided explanation with regards to input $x$, which leaves room
  for malicious providers. One such provider can just "invent" a legit
  explainer whose decision matches the discriminative one for input
  $x$. Fortunately, this technique can be detected. This detection is
  however difficult in practice, as we will illustrate next.}

\section{Illustration and Experiments}
\gilles{In this section, we instantiate concretely some of the points
  raised by our theoretical perspective.  We first illustrate the ease
  of finding PR attack functions on binary decision trees by
  presenting an algorithm that implements a PR attack. We then focus
  on detection, and evaluate in practice the hardness of finding
  incoherent pairs on the German Credit dataset.}

\subsection{Illustration using Decision Trees}
\label{s:dt}

\begin{figure*}[h]
  \begin{tikzpicture}[scale=1]
    \tikzstyle{test} = [ draw, rectangle, minimum
    width=3cm, text width=2cm, align=center, ]
    \tikzstyle{testL} = [ draw, fill=blue!10,rectangle, minimum
    width=3cm, text width=2cm, align=center, ]
    \tikzstyle{testD} = [ draw,color=red, fill=orange!10,rectangle, minimum
    width=3cm, text width=2cm, align=center, ]

    \tikzstyle{labY} = [ draw, circle, fill=green!30, node distance=10pt, minimum
    width=1cm, minimum height=0.5cm, align=center]
    \tikzstyle{labN} = [ draw, circle, fill=red!30, node distance=10pt, minimum
    width=1cm, minimum height=0.5cm, align=center, ]
    \tikzset{edge/.style={ ->,  thick }, };

    \node[testL] (dis) at (0,0) {Disguised ?};
    \node[above, xshift=-3cm] (label) {$C(x):$};
    \node[testD, below= of dis, xshift=-2cm] (age)  {Age $<60$};
    \node[testL, below= of dis, xshift=2cm] (socks)  {Wears pink socks ?};

    \node[labY,below= 1cm of age, xshift=-1cm](e1) {{\scriptsize Enter}};
    \node[labN,below= 1cm of age, xshift=1cm](b1) {{\scriptsize Bounce}};

    \node[labY,below= 1cm of socks, xshift=-1cm](e2) {{\scriptsize Enter}};
    \node[labN,below= 1cm of socks, xshift=1cm](b2) {{\scriptsize Bounce}};

    \draw[->,-{Latex[length=3mm]}] (dis) -- node[above] {Y} (age);
    \draw[->,-{Latex[length=3mm]}] (dis) -- node[above] {N} (socks);

    \draw[->,-{Latex[length=3mm]}] (age) -- node[above,xshift=-2mm] {Y} (e1);
    \draw[->,-{Latex[length=3mm]}] (age) -- node[above,xshift=2mm] {N} (b1);

    \draw[->,-{Latex[length=3mm]}] (socks) -- node[above,xshift=-2mm] {Y} (e2);
    \draw[->,-{Latex[length=3mm]}] (socks) -- node[above,xshift=2mm] {N} (b2);

    \begin{scope}[xshift=240,scale=0.6, every node/.append style={transform shape}]
      \node[testL] (dis) at (0,0) {Disguised ?};
      \node[above, xshift=-3.6cm] (label) {\Large $C'(x_l)|x_d<60:$};
      \node[testL, below= of dis, xshift=2cm] (socks)  {Wears pink socks ?};

    \node[dashed,labY,below= 1cm of dis, xshift=-1cm](e1) {{\scriptsize Enter}};

    \node[labY,below= 1cm of socks, xshift=-1cm](e2) {{\scriptsize Enter}};
    \node[labN,below= 1cm of socks, xshift=1cm](b2) {{\scriptsize Bounce}};

    \draw[->,-{Latex[length=3mm]}] (dis) -- node[above] {N} (socks);
    \draw[->,-{Latex[length=3mm]}] (dis) -- node[above,xshift=-2mm] {Y} (e1);

    \draw[->,-{Latex[length=3mm]}] (socks) -- node[above,xshift=-2mm] {Y} (e2);
    \draw[->,-{Latex[length=3mm]}] (socks) -- node[above,xshift=2mm] {N} (b2);
    \end{scope}

    \begin{scope}[xshift=240,yshift=-100,scale=0.6, every node/.append style={transform shape}]
      \node[testL] (dis) at (0,0) {Disguised ?};
      \node[above, xshift=-3.6cm] (label) {\Large $C'(x_l)|x_d \geq 60:$};     
      \node[testL, below= of dis, xshift=2cm] (socks)  {Wears pink socks ?};      
      \node[dashed,labN,below= 1cm of dis, xshift=-1cm](b1) {{\scriptsize Bounce}};
      
      \node[labY,below= 1cm of socks, xshift=-1cm](e2) {{\scriptsize Enter}};
      \node[labN,below= 1cm of socks, xshift=1cm](b2) {{\scriptsize Bounce}};

    \draw[->,-{Latex[length=3mm]}] (dis) -- node[above] {N} (socks);
    \draw[->,-{Latex[length=3mm]}] (dis) -- node[above,xshift=-2mm] {Y} (b1);

    \draw[->,-{Latex[length=3mm]}] (socks) -- node[above,xshift=-2mm] {Y} (e2);
    \draw[->,-{Latex[length=3mm]}] (socks) -- node[above,xshift=2mm] {N} (b2);
  \end{scope}
    
  \end{tikzpicture}
  \caption{Illustration of a possible implementation
    (Algorithm~\ref{alg:tree}) of the PR attack: instead of having to
    explain the use of a discriminative feature (age in this case) in the
    classifier $C$, two non-discriminative classifiers ($C'$, on the right) are
    derived. Depending on the age feature in the request, a $C'$ is
    then selected to produce legit explanations.  The two dashed
    circles represent an Incoherent Pair, that users might seek to
    detect a malicious explanation.  }
  \label{fig:trees}
\end{figure*}
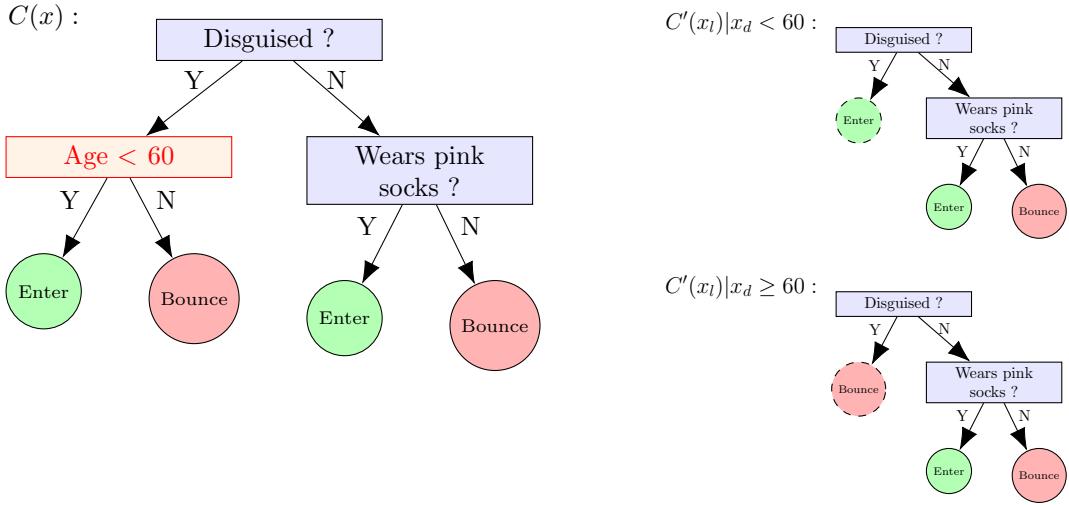

In this section, we embody the previous observations and approaches on
the concrete case of decision tree classifiers. The
choice of decision trees is motivated first because of its recognised
importance (\eg \texttt{C4.5} ranked number one of the top ten data
mining algorithms \cite{wu2008top}). Second, there is a wide consensus
on their explainability, that is straightforward \cite{molnar2019}: a
path in the tree "naturally" lists the attributes considered by the
algorithm to establish a classification. Finally, the simplicity of
crafting PR attacks on those make them good candidates for an
illustration and argues for the practical implementability of such an attack.

We denote $\mathcal{T}$ as the set of tree-based classifiers.  We do
not need any assumption on how the tree is built (\eg \texttt{C4.5}
\cite{quinlan2014c4}). Regarding explainability, we here only need to
assume that decision trees are explainable:
$ \forall C \in \mathcal{T}, exp_C$ exists.

Let $C \in \mathcal{T} \cap \mathcal{C}_d$ be a discriminatory binary
tree classifier. Each internal node $n \in V(C)$ tests incoming
examples based on a feature $n.label$. Each internal node is connected
to exactly two sons in the tree, named $n.r$ and $n.l$ for
\emph{right} and \emph{left}.  Depending on the (binary) result of
this test, the example will continue on either of these paths. We
denote the father of an internal node by $n.father$ (the root node $r$
is the only node such that $r.father=\emptyset$).

Algorithm~\ref{alg:tree} presents a PR attack on binary decision
trees. To ease its presentation, we assume that given an
input $x$, $n.r$ (right) will by convention always be the branch
taken after evaluating $x$ on $n$.
The algorithm starts by initializing the target decision
tree $C'$ as a copy of $C$. Then, it selectively removes all nodes
involving discriminative features, and replaces them with the subtree
the target example $x$ would take.

\begin{algorithm}%
\small
  \KwIn{$C$, $x=(x_l,x_d)$}
$y=C(x)$ \tcp*{Find discriminative decision }
Let $\{n_0,\ldots n_l\}$ be breadth first ordering of the nodes of $C$\;
Let $C'=C$ \tcp*{Initialise surrogate}

\For{node $i=0$ to $l$ }{
  \uIf{$n_i$.label $\in X_d$}{
    $C'.n_i.father.r= n_i.r$ \tcp*{Reconnect $n_i$ father to right son}
    $C'=C'\setminus \{n_i\}$     \tcp*{Remove discriminating node}
    $C'=C'\setminus \{n_i.l $ subtree$ \}$     \tcp*{Remove left subtree}
  }
  \Else{  
    $C'.n_i.l = \bar{y}$ \tcp*{Keep legit node, add dummy terminal node}
  }
  }
\Return $y,exp_{C'}(y,(x_l,\emptyset))$
\caption{\small PR attack on a discriminative decision binary tree $C$}
\label{alg:tree}
\end{algorithm}

To do so, Algorithm~\ref{alg:tree} removes each discriminative node
$n_i$ by connecting $n_{i-1}$ and $n_{i+1}$. While this approach would
be problematic in the general case (we would loose the $n_i.l$
subtree), in the context of $x$ we know the explored branch is
$n_i.r$, so we simply reconnect this branch, and replace the left
subtree by a dummy output.

An example is presented in Figure \ref{fig:trees}: the discriminative
classifier $C$ is queried for the explanation $exp_C(C(x),x)$ of input
$x$. To produce an answer for a discriminative feature such as the
age, it first applies Algorithm \ref{alg:tree} on $C$, given the query
$x$. If $x<60$ (upper right in Figure \ref{fig:trees}), the
explanation $exp_{C'}$ has simply became a node with the age limit,
leading to an ``Enter'' decision. In case $x\geq 60$, the explanation
node is a legit one (``Disguised''), leading to the ``Bounce''
decision. Both explanation then do not exhibit the fact that the
provider relied on a discriminative feature in $C$.  This exhibits
that $BP(1)$ does not have a solution.

\gilles{Comparing both versions of $C'$ easily yields solutions for
  $BP(2)$, for instance $(Disguised,with socks,Age=49)$ and
  $(Disguised,with socks,Age=62)$.}

\begin{lem}
  Algorithm~\ref{alg:tree} implements a PR attack.
\end{lem}
\begin{proof}
  To prove the statement, we need to prove that:
  \begin{itemize}
  \item $C'$ is legitimate
  \item $C'$ is coherent: $C'(x_l)=y$
  \item $C'$ is explainable
  \end{itemize}

  First, observe that any nodes of $C'$ containing discriminative
  features is removed line 7. \erwan{Thus, $C'$ only takes decisions
  based on features in $X_l$}: $C'$ is legitimate.

  Second, observe that by construction since $x=(x_l,x_d)$, and since
  any discriminative node $n$ is replaced by this right ($n.r$)
  outcome which is the one that corresponds to $x_d$. In other words,
  $\forall x'_l \in X_l, C'(x'_l)=C((x'_l,x_d))$: $C'$
  behaves like $C$ where discriminative features are evaluated at
  $x_d$. This is true in particular for $x_l:
  C'(x_l)=C((x_l,x_d))=C(x)=y$.

  Finally, observe that $C'$ is a valid decision tree. Therefore,
  according to our explainability framework, $C'$ is explainable.
\end{proof}

Interestingly, the presented attack can be efficient as it only
involves pruning part of the target tree. In the worst case, this one
has $\Omega(2^{d})$ elements, but in practice decision trees are
rarely that big.

\subsection{Finding IPs on a Neural Model: the German Credit Dataset}
\label{s:nn}

We now take a closer look at the detectability of the attack, namely:
how difficult is it to spot an IP ?
We illustrate this by experimenting on the German Credit Dataset.

\paragraph{Experimental setup}

We leverage Keras over TensorFlow to learn a neural network-based
model for the German Credit Dataset \cite{credit}. \erwan{
While we could have used any relevant type of classifier for our experiments, general current focus is on neural networks regarding explainability. 
  The bank dataset
}classifies client profiles ($1,000$ of them), described by a set of
attributes, as good or bad credit risks. Multiple techniques have been
employed to model the credit risks on that dataset, which range from
$76.59\%$ accuracy for a SVM to $78.90\%$ for a hybrid between genetic
algorithm and a neural network \cite{ORESKI20142052}.

The dataset is composed of 24 features (some categorical ones, such as
sex, of status, were set to numerical). This thus constitutes a low
dimensional dataset as compared to current applications (observations
in \cite{andreou2018ndss} reported up to $893$ features for the sole
application of ad placement on user feeds on
Facebook). \erwan{Furthermore, modern classifiers are currently
  dealing with up to $512 \times 512 \times 3$ dimensions
  \cite{brock2018large}, which permit a significant increase in data
  processing and thus the capability to expand the amount of features
  taken into account for decision-making.}

The neural network we built\footnote{Code is made available at:
  \url{https://github.com/erwanlemerrer/bouncer_problem}.}  is
inspired by the one proposed \cite{KHASHMAN20106233} in 2010, and that
reached $73.17\%$ accuracy. It is a simple multi-layer perceptron,
with a single hidden layer of $23$ neurons (with sigmoid activations),
and a single output neuron for the binary classification of the input
profile to ``risky'' or not. In this experiment we use the Adam
optimizer and a learning rate of $0.1$ (leading to much faster
convergence than in \cite{KHASHMAN20106233}), with a validation split
of $25\%$.  We create 30 models, with an average accuracy of
$76.97\%@100$ epochs on the validation set (with a standard deviation
of $0.92\%$).

\erwan{In order to generate input profiles, we consider two scenarios.
  In A) we consider a scenario where a user sets a random value in a
  discriminative feature to try to find an IP. This yields rather
  artificial user profiles (that may be detected as such by the remote
  service provider). To have an aggregated view of this scenario, we
  proceed as follows. For each of the 30 models, we randomly select 50
  users as a test set (not used for training the previous models). We
  then repeat 500 times the following: we select a random user among
  the 50 and select a random discriminative feature among the four, to
  set a random (uniform) value in it (belonging to the domain of each
  selected feature, \eg from 18 to 100 in the age feature). This
  creates a set of $15,000$ fake profiles as inputs.
}

  \erwan{ In B), in order to have a more realistic scenario where
    profiles are created from real data from the dataset, we now
    proceed as follows.}  We also select $50$ profiles from the
  dataset as a test set, so we can perform our core experiment: the
  four discriminatory features of each of those profiles are
  sequentially replaced by the ones of the 49 remaining profiles; each
  resulting test profile is fed to the model for prediction. (This
  permits to test the model with realistic values in those
  features. This process creates 2450 profiles to search of an IP). We
  count the number of times the output risk label has switched, as
  compared to the original untouched profile fed to the model.  We
  repeat this operation on the 30 models to observe deviations.

\paragraph{The low probability of findings IPs at random}

\begin{figure}
\includegraphics[width=0.9\linewidth]{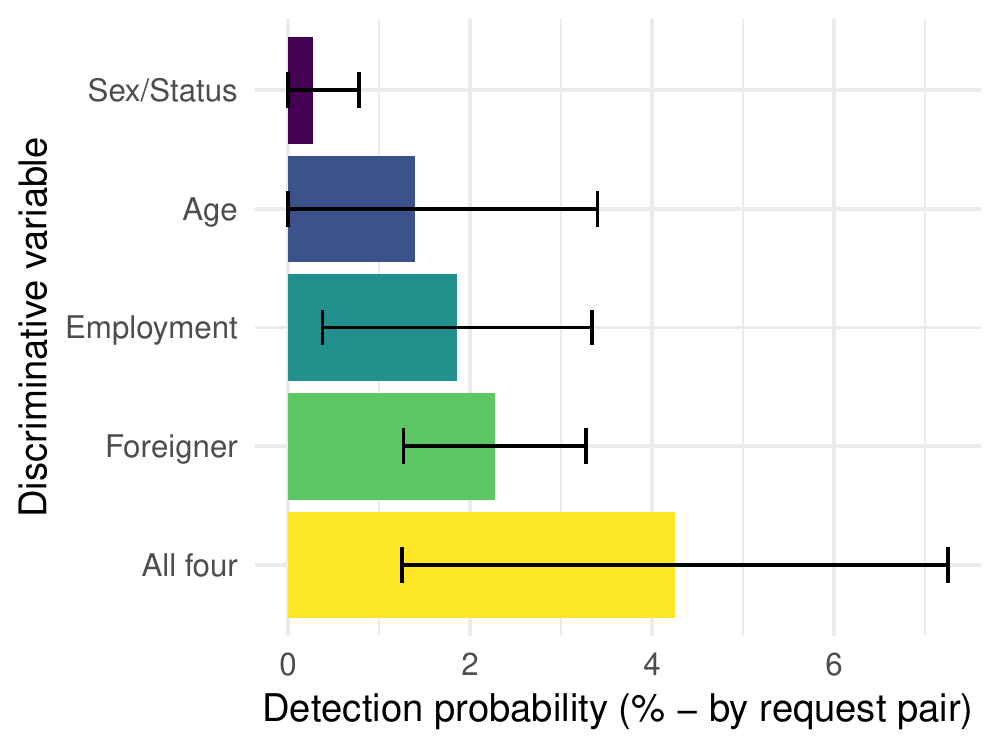}
  \caption{Percentage of label changes when swapping the discriminative features in the test set data for scenario B). \erwan{Bars indicate standard deviations.} Those indicate the low probability to spot a \textbf{PR} attack on the provider model).}
\label{f:credit}
\end{figure}

\erwan{In the case of scenario A), we compare the original label with
  the one obtained from each crafted input. We obtain $8.09\%$ of IPs
  (standard deviation of $4.08$).}

Figure \ref{f:credit} depicts for scenario B) the proportion of label
changes over the total number of test queries; recall that a label
change while considering two inputs constitutes an IP.  We observe
that if we just change one of the four features, we obtain on average
$1.86\%$, $0.27\%$, $1.40\%$, $2.27\%$ labels changes (for the
employment, sex/status, age, foreigner features, respectively), while
$4.25\%$ if the four features are simultaneously changed. (Standard
deviations are of $1.48\%$, $0,51\%$, $1,65\%$, $2,17\%$ and $3,13\%$,
respectively).

\erwan{This probability of $4.25\%$ is higher than our deterministic
lower bound $BP(\vert \mathcal{X}\vert)$ (Proposition \ref{lower}),
hinting that this discriminating classifier is easier to spot that the
worst-case one.  }
Moreover, since not finding an IP after some requests does not guarantee
the absence of a discriminating behaviour, we now look at the user-side
perspective: testing the absence of discrimination of a remote service.
It turns out that we can
compute an expectation of the number of queries for such a user to find
an IP.

Users can query the service with inputs, until they are
\emph{confident} enough that such pair does not exist.
Assuming one seeks a $99\%$ confidence level --that is, less than one
percent of chances to falsely detect a discriminating classifier as
non-discriminating--, and using the detection probabilities of Figure
\ref{f:credit}, we can compute the associated p-values. A user testing
a remote service based on those hypotheses would need to craft
respectively $490, 2555, 368, 301$ and $160$ (for the employment,
sex/status, age, foreigner, and all four respectively) 
pairs in the hope to decide on the existence or not of an IP, as presented in Figure \ref{f:creditConf} (please note the log-scale on the $y$-axis).

\erwan{Those experiments highlight} the hardness to experimentally
check for PR attacks.

\begin{figure}
\includegraphics[width=\linewidth]{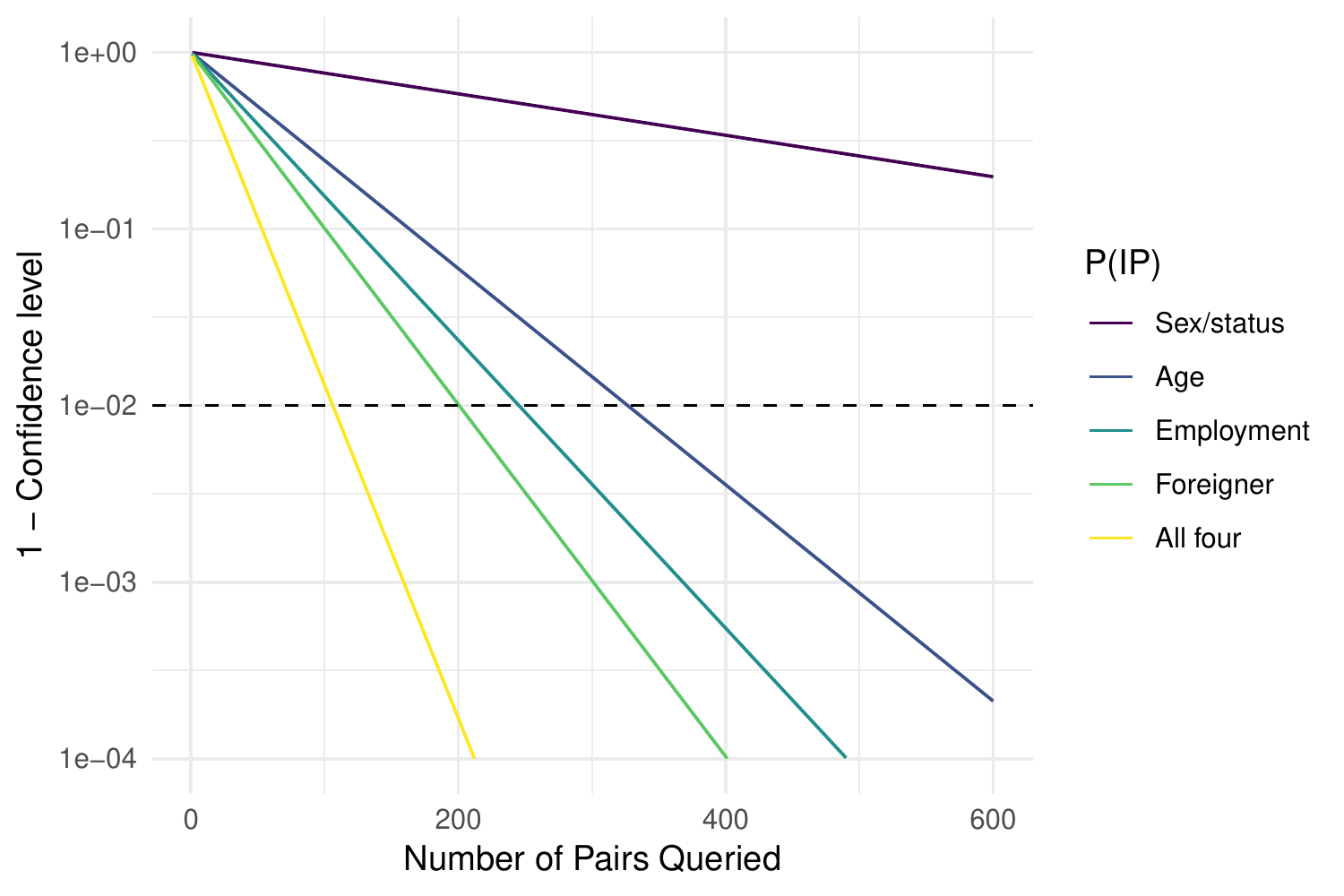}
  \caption{Confidence level as a function of the number of tested
    input pairs, based on the German Credit detection probability in Figure
    \ref{f:credit}. The dashed line represents a $99\%$ confidence level.}
\label{f:creditConf}
\end{figure}

\section{Discussion}
\label{s:discussion}

We now list in this section several consequences of the findings in
this paper, and some open questions.

\subsection{Findings and Applicability}

We have shown that a malicious provider can always craft a fake
explanation to hide its use of discriminatory features, by creating a
surrogate model for providing an explanation to a given user.  An
impossibility result follows, for a user to detect such an attack
while using a single explanation.  The detection by a user, or a group
of users, is possible only in the case of multiple and deliberate
queries ($BP(\epsilon>1)$); this process may require an exhaustive
search of the input space.

\gilles{However, we see that our practical experiment on the German
  Credit Dataset is far from this bound. Intuitively the probability
  of finding an IP is proportional to the "discrimination level" of a
  classifier. While quantifying such level is a difficult task, we
  explore a possible connection in the next paragraph.}

We note that the malicious providers have another advantage for
covering PR attacks. Since multiple queries must be issued to spot
inconsistencies via IP pairs, basic rate limiting mechanisms for
queries may block and ban the incriminated
users. Defenses of this kind, for preventing attacks on online machine services exposing APIs, are being proposed \cite{Hou:2019:MDA:3326285.3329042}.
This adds another layer of complexity for the observation of
misbehaviour.

\subsection{Connection with Disparate Impact}

We now briefly relate our problem to \textit{disparate impact}: a recent
article \cite{feldman2015certifying} proposes to adopt "a
generalization of the 80 percent rule advocated by the US Equal
Employment Opportunity Commission (EEOC)" as a criteria for
disparate impact. This notion of disparate impact proposes to
capture discrimination through the variation of outcomes of an
algorithm under scrutiny when applied to different population groups. 

More precisely, let $\alpha$ be the disparity
ratio. The authors propose the following formula, here adapted to
our notations \cite{feldman2015certifying}:
$$\alpha= \frac{\mathbb{P}(y|x_d=0)}{\mathbb{P}(y|x_d=1)},$$ where
\erwan{$X_d=\{0,1\}$ is the discriminative space reduced to a
binary discriminatory variable}. Their approach is to consider that if
$ \alpha<0.8$ then the tested algorithm could be qualified as discriminative.

To connect disparate impact to our framework, we can conduct the
following strategy. \erwan{Consider a classifier
$C$ having a disparate impact $\alpha$, and producing a binary decision
$C(x) \in \{ 0 = $ ``bounce'', $1 = $ ``enter''$\}$.} We search for Incoherent
Pairs as follows: first, pick $x\in X_l$ a set of legit features. Then
take $a=(x,x_d=0)$, representing the discriminated group, and
$b=(x,x_d=1)$ representing the undiscriminated group.  Then test $C$ on both
$a$ and $b$: if $C(a)\neq C(b) $ then $(a,b)$ is an IP. \erwan{The
probability $\mathbb{P}$} of finding an IP in this approach can be
written as $\mathbb{P}(IP)$. Let $A$ (resp. $B$) be the event ``$a$
enters'' (resp. ``$b$ enters'').

We can develop:\begin{align*} \mathbb{P}(IP) &= \mathbb{P}(C(a)\neq C(b)) \\
                               &=\mathbb{P}(A\cap\overline{B})+\mathbb{P}(\overline{A}\cap B) \\
                               &= \mathbb{P}(A) -\mathbb{P}(A \cap B) + \mathbb{P}(B) - \mathbb{P}(A \cap B)\\
                               &=\mathbb{P}(B)(1+\alpha) -2
                                 \mathbb{P}(A\cap B) \text{, since } \alpha= \mathbb{P}(A)/\mathbb{P}(B).
               \end{align*}
Using conditional probabilities, we have $\mathbb{P}(A\cap B) =
\mathbb{P}(B|A).\mathbb{P}(A)$. Thus
$\mathbb{P}(IP)= \mathbb{P}(B)(1+\alpha -2\alpha.\mathbb{P}(B|
A))$. Since the conditional probability $\mathbb{P}(B|
A)$ is difficult to assess without further hypotheses on $C$, let
us investigate two extreme scenarios:
\begin{itemize}
\item Independence: $A$ and $B$ are completely independent
  events, even though $a$ and $b$ share their legit features in $X_l$. This
  scenario, which is not very realistic, could model purely random
  decisions with respect to attributes from $X_d$. In this scenario $\mathbb{P}(B|
A) = \mathbb{P}(B)$.
\item Dependence: $A \Rightarrow B$: if $A$ is selected despite
  its membership to the discriminated group ($a=(x,0)$), then
  necessarily $b$ must be selected, as it can only be ``better'' from
  $C$'s perspective.  In this scenario $\mathbb{P}(B|
A) = 1$.
\end{itemize}

\begin{figure}[h]
  \includegraphics[width=1.05\linewidth]{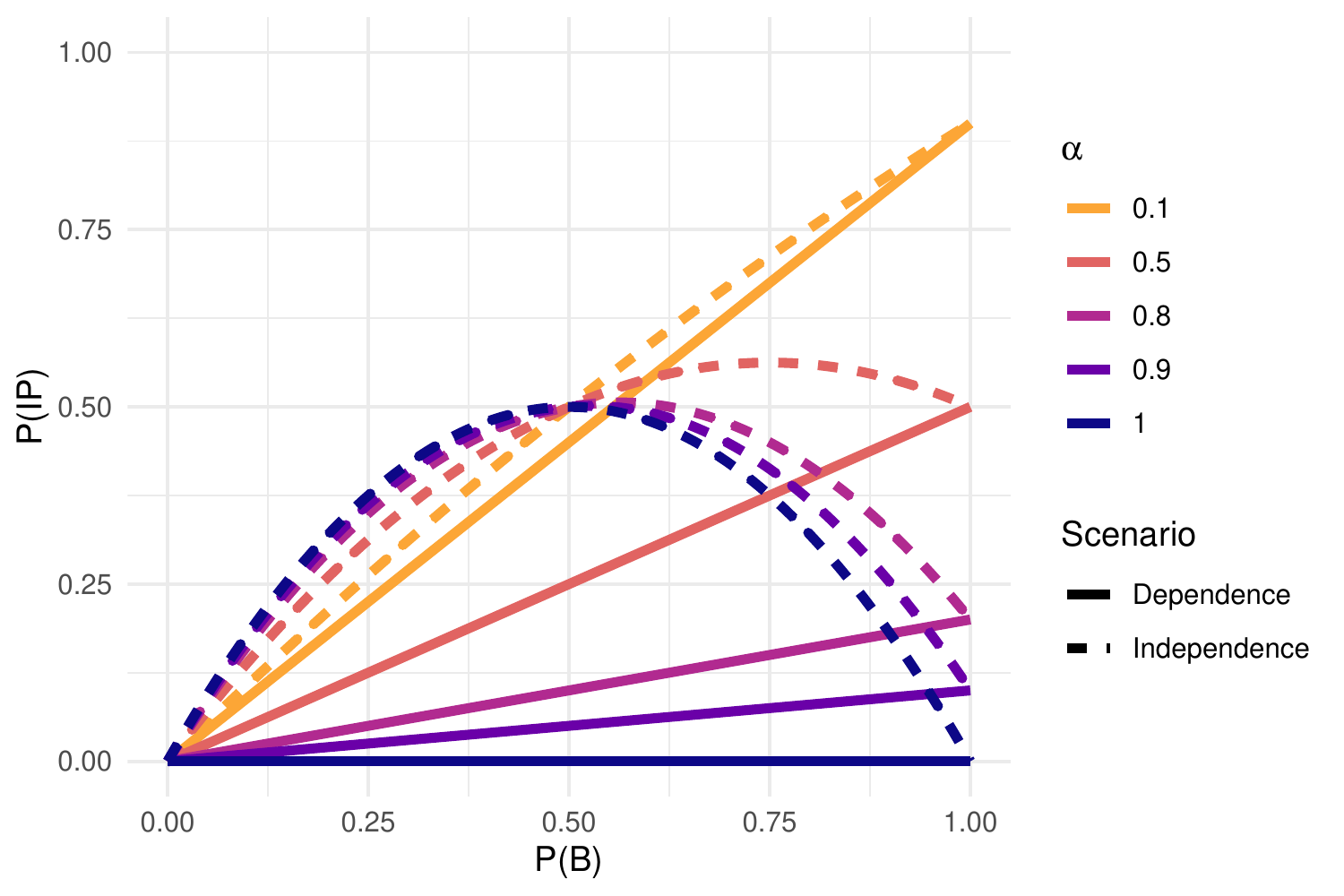}
  \caption{Probability to find an Incoherent Pair (IP), as a function
    of $\mathbb{P}(B)$ the probability of success for a
    non-discriminated group. \erwan{$\alpha$ represents the disparity ratio.}}
  \label{fig:disimp}
\end{figure}
Figure~\ref{fig:disimp} represents the numerical evaluation of our two
scenarios. First, it shows that the probability of finding an IP
strongly depends on the probability of a success for the
non-discriminated group $\mathbb{P}(B)$. Indeed, since the
discriminated group has an even lower probability of success, a low
success probability for the non-discriminated group implies frequent
cases where both $a$ and $b$ are failures, which does not
constitutes an IP.

In the absence of disparate impact ($\alpha=1$), both scenarios provide
very different results: the independence scenario easily
identifies IPs --which is coherent with the "random"
nature of the independence assumption--. This underlines the unrealistic
nature of the independence scenario in this context. With a high disparate impact
however (\eg $\alpha=0.1$), the discriminated group has a high 
probability of failure. Therefore the probability of finding an IP is very close
to the simple probability of the non-discriminated group having a success
$\mathbb{P}(B)$, regardless of the considered scenario.

The dependence scenario nicely illustrates a natural connection: the
higher the disparate impact, the higher the probability to find an IP.
While this only constitutes a thought experiment, we believe this
highlights possible connections with standard discrimination measures
and conveys the intuition that in practice, the probability of finding
IPs exposing a PR attack strongly depends on the
intensity of the discrimination hidden by that PR attack.

\subsection{Open Problems for Remote Explainability}

\paragraph{On the test efficiency}

It is common for fairness assessment tools to leverage testing. As
the features that are considered discriminating are often precise
\cite{10.1007/978-3-642-22589-5_20,Galhotra:2017:FTT:3106237.3106277},
the test queries for fairness assessment can be targeted and some
notions of efficiency in terms of the amount of requests can be
derived. This may be done by sampling the feature space
under question for instance (as in work by Galhotra et
al. \cite{Galhotra:2017:FTT:3106237.3106277}).

Yet, it appears that with current applications such as social
networks \cite{andreou2018ndss}, users spend a considerable amount of time online, producing
more and more data that turn into features, and also are the basis to
the generation of other meta-features. In that context, the full scope
of features, discriminating or not, may not be clear to a user. This
makes exhaustive testing even theoretically unreachable, due to the
very likely non-complete picture of what providers are using to issue
decisions. This is is another challenge on the way to remote
explainability, if providers are not willing to release a complete
and precise list of all attributes leveraged in their system.

\paragraph{Towards a provable explainability?}
Some other computing applications, such as data storage or intensive
processing also have questioned the possibility of malicious service
providers in the past. Motivated by the plethora of offers in the
cloud computing domain and the question of quality of service,
protocols such as \textit{proof of data possession} \cite{pdp}, or
\textit{proof-based verifiable computation} \cite{comp-verif}, assume
that the service provider might be malicious. A solution to still have
services executed remotely in this context is then to rely on
cryptographic protocols to formally verify the work performed
remotely. To the best of our knowledge, no such provable process
logic has been adapted to explainability. That is certainly an
interesting development to come.

\section{Related Work}
\label{s:related}

\paragraph{Explaining in-house models}
As a consequence of the major impact of machine learning models in
many areas of our daily life, the notion of \textit{explainability}
has been pushed by policy makers and regulators.  Many works address
explainability of inspected model decisions on a local setup (please
refer to surveys
\cite{guidotti2018survey,datta2016algorithmic,molnar2019}) --some
specifically for neural network models \cite{NIPS2018_8141}--, where
the number of requests to the model is unbounded. Regarding the
question of fairness, a recent work specifically targets the fairness
and discrimination of in-house softwares, by developing a
testing-based method \cite{Galhotra:2017:FTT:3106237.3106277}.

\paragraph{Dealing with remote models}
The case of models available through a remote black-box interaction
setup is particular, as external observers are bound to scarce data
(labels corresponding to inputs, while being limited in the number of
queries to the black-box \cite{tramer}).  Adapting the explainability
reasoning to models available in a black box setup is of a major
societal interest: Andreou et al. \cite{andreou2018ndss} shown that
Facebook's explanations for their ad platform are incomplete and
sometimes misleading. They also conjecture that malicious service
providers can ``hide'' sensitive features used, by explaining
decisions with very common ones. In that sense, our paper is exposing
the hardness of explainability in that setup, confirming that
malicious attacks are possible.  Milli et
al. \cite{Milli:2019:MRM:3287560.3287562} provide a theoretical ground
for reconstructing a remote model (a two-layer ReLu neural network)
from its explanations and input gradients; if further research proves
the approach practical for current applications, this technique may
help to infer the use of discriminatory features in use by the service
provider.

\paragraph{Operating without trust: the domain of security}
In the domain of security and cryptography, some similar setups have
found a large body of work to solve the trust problem in remote
interacting systems. In \textit{proof of data possession} protocols
\cite{pdp}, a client executes a cryptographic protocol to verify the
presence of her data on a remote server; the challenge that the
storage provider responds to assesses the possession or not of some
particular piece of data. Protocols can give certain or probabilistic
guarantees.  In \textit{proof-based verifiable computation}
\cite{comp-verif}, the provider returns the results of a queried
computation, along with a proof for that computation. The client can
then check that the computation indeed took place.  These schemes,
along with this paper exhibiting attacks on remote explainability, motivate
the need for the design of secure protocols.

\paragraph{Discrimination and bias detection approaches}
Our work is complementary to classic discrimination detection in
automated systems. In contrast to works on \textit{fairness} \cite{fair-phil} that attempt to
identify and measure discrimination from systems, our work does not
aim at spotting discrimination, as we have shown it can be hidden by
the remote malicious provider. We instead are targeting the occurrence
of incoherent explanations produced by such a provider in the will to
cover its behavior, which is a a completely different nature than
fairness based test suites. Galhotra et
al. \cite{Galhotra:2017:FTT:3106237.3106277}, inspired by statistical
causality \cite{CIS-247618}, for instance propose to create input
datasets for observing discrimination on some specific features by the
system under test.

\erwan{While there are numerous comments and proposals for good
  practices when releasing models that may include forms of bias
  \cite{10.1145/3287560.3287596}, the automatic detection of bias on
  the user side is also of interest for the community.  For instance,
  researchers seek to detect the Simpson's Paradox
  \cite{10.2307/2284382} in the data
  \cite{alipourfard2018using}. Another work makes use of
  \textit{causal graphs} to detect \cite{10.1145/3097983.3098167} a
  potential discrimination in the data, while authors propose in
  \cite{6175897} to purge the data so that direct and/or indirect
  discriminatory decision rules are converted to legitimate
  classification rules. Some works are specific to some applications, such as financial ones \cite{zhang2019fairness}.
  Note that those approaches by
  definition require an access to the training data, which is a too
  restrictive assumption in the context of our target contribution.
}

\erwan{The work in \cite{tan2017distillandcompare} proposes to
  leverage transfer learning (or \textit{distillation}) to mimic the behaviour
  of a black box model, here a credit scoring model. A collection
  campaign is assumed to provide a labeled dataset with the risk
  scores, as produced by the model and the ground-truth outcome. From
  this dataset is trained a model aiming at mimicking the black box as
  close as possible. Both models are then compared on their outcome,
  and a method to estimate the confidence interval for the variance of
  results is presented. The trained model can then be queried to
  assess potential bias. This approach proves solid guarantees when
  one assumes that the dataset is extracted from a black box that does
  not aim to bias its outputs to prevent audits of that form.}

\paragraph{The rationalization of explanations}
More closely related to our work is the recent paper
by Aivodji et al. \cite{pmlr-v97-aivodji19a}, that introduces the concept of
rationalization, in which a black-box algorithm is approximated by a
surrogate model that is "fairer" that the original black-box. In our
terminology, they craft $C'$ models that optimise arbitrary fairness
objectives. To achieve this, they explore decision tree models trained
using the black-box decisions on a predefined set of inputs.
This produces another argument against black-box explainability in a remote context.
The main technical difference with our tree algorithm section
\ref{s:dt} is that their surrogates $C'$ optimises an exterior metric
(fairness) at the cost of some coherence (fidelity in the authors'
terminology). In contrast, our illustration section
\ref{s:dt} produces surrogates with perfect coherence that do not
optimise any exterior metric such as fairness. In our model, spotting
an incoherence (\ie the explained model produces a $y$ while the
black-box produces a $\bar{y}$) would directly provide a proof of
manipulation and reveal the trickery.
Interestingly, the incoherent pair approach fully applies in the context of
their model surrogates, as it arises as soon as more than one
surrogate is used (regardless of the explanation).  %
Our paper focuses on the user-side observation of explanations, and
users ability to discover such attacks. We rigorously prove that
single queries are not sufficient to determine a manipulation, and
that the problem is hard even in the presence of multiple queries and
observations. %

\section{Conclusion}
\label{s:conclusion}

In this paper, we studied explainability in a remote context, which is
sometimes presented as a way to satisfy society's demand for
transparency facing automated decisions. We prove it is unwise to
blindly trust those explanations:
  like humans, algorithms can easily
  hide the true motivations of a decision when asked for explanation.
To that end, we presented an attack that generates
  explanations to hide the use of an arbitrary set of features by a
  classifier. While this construction applies to any classifier
  queried in a remote context, we also presented a concrete
  implementation of that attack on decision trees.  
On the defensive side, we have shown that such a manipulation cannot
be spotted by one-shot requests, which is unfortunately the nominal
use-case. However, the proof of such trickery (pairs of
classifications that are not coherent) necessarily exists. We further
evaluated in a practical scenario the probability of finding such
pairs, which is low. The attack is thus arguably impractical to detect
for an isolated user.

We conclude that this must consequently question the whole concept of
the explainability of a remote model operated by a third party
provider, at the very least. A research direction is to develop secure
schemes in which the involved parties can trust the exchanged
information about decisions and their explainability, as enforced by
new protocols. A second line of research may be the collaboration of
users observations for spotting the attack in an automated
way. \erwan{Indeed, as done by Chen et
  al. \cite{10.1145/2815675.2815681} for understanding the impact of
  Uber surge pricing on passengers and drivers (by deploying 43 Uber
  accounts that act as clients), data can be put in common to retrieve
  information more reliably. The anonymization of users data if a pool of measurements is made public is for sure a crucial point to ensure scalable observation of black box decision-making systems.}
We
believe this is an interesting development to come, in relation with
the promises of AI and automated decisions processes.

\bibliographystyle{abbrv}
{\small
\bibliography{biblio}
}

\end{document}